\pdfoutput=1
\documentclass{article}
\PassOptionsToPackage{finalizecache=false,frozencache=true}{minted}
\usepackage{arxiv}
\usepackage{microtype}
\usepackage{graphicx}
\usepackage{subfigure}
\usepackage{booktabs}
\usepackage{subfigure}
\usepackage[colorlinks=true,linkcolor=black,citecolor=blue,urlcolor=blue]{hyperref}

\usepackage{amsmath}
\usepackage{amssymb}
\usepackage{mathtools}
\usepackage{amsthm}
\usepackage{bm}
\usepackage[noframes,nobars]{ffcode}
\usepackage{nicefrac}
\usepackage{siunitx}
\usepackage{natbib}
\setcitestyle{numbers ,square}
\usepackage[capitalize,noabbrev]{cleveref}

\usepackage[acronym]{glossaries}
\usepackage{glossaries-extra}
\setabbreviationstyle{long-short-sc}

\newabbreviation[
  longplural={Gaussian Processes}
]{gp}{gp}{Gaussian Process}
\newabbreviation[
]{clt}{clt}{Central Limit Theorem}
\newabbreviation[
  longplural={Hilbert--space Gaussian Processes}
]{hgp}{hgp}{Hilbert--space Gaussian Process}
\newabbreviation[
  longplural={Tensor Networks}
]{tn}{tn}{Tensor Network}
\newabbreviation[
]{cpd}{cpd}{Canonical Polyadic Decomposition}
\newabbreviation[
]{tt}{tt}{Tensor Train}
\newabbreviation[
]{bf}{bf}{Basis Function}
\newabbreviation[
]{ml}{ml}{maximum likelihood}
\newabbreviation[
]{map}{map}{maximum a posteriori}
\newabbreviation[
]{dnn}{dnn}{Deep Neural Network}
\newabbreviation[
]{cnn}{cnn}{Convolutional Neural Network}
\newabbreviation[
]{rnn}{rnn}{Recurrent Neural Network}
\newabbreviation[
]{gnn}{gnn}{Graph Neural Network}
\newabbreviation[
]{mera}{mera}{Multi-Scale Entanglement Renormalization Ansatz}
\newabbreviation[
]{pdf}{pdf}{Probability Density Function}
\newabbreviation[
]{cdf}{cdf}{Cumulative Density Function}
\newabbreviation[
]{rmse}{rmse}{Root Mean Squared Error}

\theoremstyle{plain}
\newtheorem{theorem}{Theorem}[section]

\newtheorem{lemma}[theorem]{Lemma}
\newtheorem{corollary}[theorem]{Corollary}
\theoremstyle{definition}
\newtheorem{definition}[theorem]{Definition}

\theoremstyle{remark}

\title{Tensor Network-Constrained Kernel Machines as Gaussian Processes}

\crefname{equation}{equation}{equations}
\Crefname{equation}{Equation}{Equations}
\crefname{theorem}{theorem}{theorem}
\Crefname{theorem}{Theorem}{Theorems}
\crefname{definition}{definition}{definitions}
\crefname{corollary}{corollary}{corollaries}
\Crefname{corollary}{Corollary}{Corollaries}
\crefname{lemma}{lemma}{lemmas}
\Crefname{lemma}{Lemma}{Lemmas}
\crefname{appendix}{appendix}{appendices}
\crefname{section}{section}{sections}
\crefname{figure}{figure}{figures}

\newcommand{\mat}[1]{\bm{#1}}
\newcommand{\ten}[1]{\bm{\mathcal{#1}}}
\newcommand{\inner}[2]{\left\langle #1, #2 \right\rangle}
\newcommand{\innerfrob}[2]{{\langle #1, #2 \rangle}_\mathrm{F}}
\newcommand{\norm}[1]{{\left\lvert\left\lvert #1 \right\rvert\right\rvert}_\mathrm{F}}
\newcommand{\vectorize}[1]{\operatorname{\ff{vec}}\left(#1\right)}
\newcommand{\tensorize}[1]{\operatorname{\ff{ten}}\left(#1\right)}
\newcommand{\kron}{\otimes}

\newcommand{\trans}{^\mathrm{T}}

\newcommand{\order}[1]{$\mathcal{O}(#1)$}
\newcommand{\eye}[1]{\text{I}_{#1}}

\newcommand{\TN}[1]{\ff{TN}(#1)}
\newcommand{\CPD}[1]{\ff{CPD}(#1)}
\newcommand{\TT}[1]{\ff{TT}(#1)}
\newcommand{\RANKONE}[1]{\ff{R}_1(#1)}

\newcommand{\hiddenCP}{h}
\newcommand{\hiddenTT}[1]{{h}^{(#1)}}
\newcommand{\hiddenTTMat}[1]{\mat{h}^{(#1)}}

\newcommand{\tempTTMat}[1]{\mat{Z}^{(#1)}}

\DeclareMathOperator{\isdef}{\coloneqq}

\DeclareMathOperator{\E}{\mathbb{E}}

\newif\ifuniqueAffiliation
\uniqueAffiliationtrue

\ifuniqueAffiliation 
\author{Frederiek Wesel\\
	Delft Center for Systems and Control\\
	Delft University of Technology\\
	The Netherlands\\
	\texttt{f.wesel@tudelft.nl} \\
	\And
	Kim Batselier\\
	Delft Center for Systems and Control\\
	Delft University of Technology\\
	The Netherlands\\
	\texttt{k.batselier@tudelft.nl} \\
}

\fi

\renewcommand{\shorttitle}{Tensor Network-Constrained Kernel Machines as Gaussian Processes}

\hypersetup{
pdftitle={Tensor Network-Constrained Kernel Machines as Gaussian Processes},
pdfauthor={Frederiek Wesel, Kim Batselier},
pdfkeywords={Tensor Networks, Gaussian Processes, GPs, Tensor Decompositions, Kernel},
}

\begin{document}
\maketitle

\begin{abstract}
    \glspl*{tn} have recently been used to speed up kernel machines by constraining the model weights, yielding exponential computational and storage savings.
    In this paper we prove that the outputs of \gls*{cpd} and \gls*{tt}-constrained kernel machines recover a \gls*{gp}, which we fully characterize, when placing i.i.d. priors over their parameters.
    We analyze the convergence of both \gls*{cpd} and \gls*{tt}-constrained models, and show how \gls*{tt} yields models exhibiting more \gls*{gp} behavior compared to \gls*{cpd}, for the same number of model parameters.
    We empirically observe this behavior in two numerical experiments where we respectively analyze the convergence to the \gls*{gp} and the performance at prediction.
    We thereby establish a connection between \gls*{tn}-constrained kernel machines and \glspl*{gp}.
\end{abstract}

\section{Introduction}
\emph{Tensor Networks} \citep[\glspl*{tn},][]{cichocki_era_2014,cichocki_tensor_2016,cichocki_tensor_2017}, a tool from multilinear algebra, extend the concept of rank from matrices to tensors allowing to represent an exponentially large object with a linear number of parameters.
As such, \glspl*{tn} have been used to reduce the storage and computational complexities by compressing the model parameters of a range of models such as \glspl*{dnn} \citep{novikov_tensorizing_2015}, \glspl*{cnn} \citep{jaderberg_speeding_2014,lebedev_speeding-up_2015}, \glspl*{rnn} \citep{ye_learning_2018}, \glspl*{gnn} \citep{hua_high-order_2022} and transformers \citep{ma_tensorized_2019-1}.

Similarly, \glspl*{tn} have also found application in the context of kernel machines \citep{stoudenmire_supervised_2016,novikov_exponential_2018,wesel_large-scale_2021}.
Such models learn a multilinear (i.e. nonlinear) data-dependent representation from an exponentially large number of fixed features by means of a linear number of parameters, and are as such characterized by an implicit source of regularization. 
Furthermore, storage and the evaluation of the model and its gradient require a linear complexity in the number of parameters, rendering these methods promising candidates for applications requiring both good generalization and scalability.
However their multilinearity precludes closed-form Bayesian inference, and has restricted the training of these models to the \gls*{ml} and \gls*{map} framework.

In contrast, \emph{Gaussian Processes} \citep[\glspl*{gp},][]{rasmussen_gaussian_2006} are an established framework for modeling functions which naturally allow the practitioner to incorporate prior knowledge.
Importantly, when considering i.i.d. observations and Gaussian likelihoods, \glspl*{gp} allow for the determination of the posterior in \emph{closed-form}, which considerably facilitates tasks such as inference, sampling and the construction of sparse approximations among many others.
The main drawback of having a closed-form posterior lies however in their inability to autonomously learn features, which has arguably favored the use of deep learning models.

In this paper we establish a connection between \gls*{tn}-constrained kernel machines and \glspl*{gp}, thus solving an open problem considered by \citet{wesel_large-scale_2021,wesel_tensor-based_2023}. We prove that the outputs of \emph{Canonical Polyadic Decomposition} (\gls*{cpd}) and \emph{Tensor Train} (\gls*{tt})-constrained kernel machines converge to a fully characterized \gls*{gp} when specifying i.i.d. priors across their components. 
This result allows us to derive that that for the same number of model parameters, \gls*{tt}-constrained models achieve faster convergence to the \gls*{gp} compared to their \gls*{cpd} counterparts and thus are more prone to exhibit \gls*{gp} behavior.
We analyze the consequences of these findings in the context of \gls*{map} estimation and finally empirically observe \gls*{gp} convergence and \gls*{gp} behavior in two numerical experiments.

The rest of this paper is organized as follows. In \cref{sec:background} we provide a brief introduction to \glspl*{gp} and their approximations, \glspl*{tn} and \gls*{tn}-constrained kernel machines. In \cref{sec:main} we present our main result, i.e. the equivalence in the limit between \gls*{tn}-constrained kernel machines and \glspl*{gp}. In \cref{sec:experiments} we discuss our numerical results which illustrate the found results. We then provide a review of related work (\cref{sec:related_work}) and a conclusion (\cref{sec:conclusion}).
We discuss the notation used throughout the paper in \cref{sec:notation}.
\section{Background}\label{sec:background}
\glspl*{gp} are a collection of random variables such that any finite subset has a joint Gaussian distribution \citep{rasmussen_gaussian_2006}. They provide a flexible formalism for modeling functions which inherently allows for the incorporation of prior knowledge and the production of uncertainty estimates in the form of a predictive distribution.
More specifically, a \gls*{gp} is fully specified by a mean function $\mu \in \mathbb{R}$, typically chosen as zero, and a covariance or kernel function $k(\cdot,\cdot):\mathbb{R}^D\times \mathbb{R}^D \rightarrow \mathbb{R}$:
\begin{equation*}\label{eq:GP}
    f(\mat{x}) \sim \mathcal{GP}(\mu,k(\mat{x},\cdot)).
\end{equation*}
Given a labeled dataset $\{(\mat{x}_n,y_n)\}_{n=1}^{N}$ consisting of inputs $\mat{x}_n\in \mathbb{R}^D$ and i.i.d. noisy observations $y_n\in\mathbb{R}$, \glspl*{gp} can be used for modeling the underlying function $f$ in classification or regression tasks by specifying a likelihood function.
For example the likelihood
\begin{equation}\label{eq:likelihood}
    p(y_n \mid f(\mat{x}_n)) = \mathcal{N}(f(\mat{x}_n),\sigma^2),
\end{equation}
yields a \gls*{gp} posterior which can be obtained in closed-form by conditioning the prior \gls*{gp} on the noisy observations.
Calculating the mean and covariance of such a posterior crucially requires instantiating and formally inverting the kernel matrix $\mat{K}$ such that $k_{n,m} \isdef k(\mat{x}_n,\mat{x}_m)$. These operations respectively incur a computational cost of \order{N^2} and \order{N^3} and therefore prohibit the processing of large-sampled datasets.

\subsection{Basis Function Approximation}
The prevailing approach in literature to circumvent the \order{N^3} computational bottleneck is to project the \gls*{gp} onto a finite number of \glspl*{bf} \citep[e.g.,][]{rasmussen_gaussian_2006,quinonero-candela_unifying_2005}. This is achieved by approximating the kernel as
\begin{equation}
    {k}(\mat{x},\mat{x'}) \approx {\mat{\varphi}(\mat{x})}\trans \mat{\Lambda} \ \mat{\varphi}(\mat{x'}),
\end{equation}
where here $\mat{\varphi}(\mat{x}):\mathbb{R}^D\rightarrow\mathbb{R}^M$ are (nonlinear) basis functions and $\mat{\Lambda}\in\mathbb{R}^{M\times M}$ are the \gls*{bf} weights.
This finite-dimensional kernel approximation ensures a \emph{degenerate} kernel \citep{rasmussen_gaussian_2006}, as it characterized by a finite number of non-zero eigenvalues.
Its associated \gls*{gp} can be characterized equivalently as
\begin{equation}\label{eq:low_rank_GP}
    f(\mat{x}) = \inner{\mat{\varphi}(\mat{x})}{\mat{w}}, \quad \mat{w}\sim \mathcal{N}(\mat{0},\mat{\Lambda}),
\end{equation}
wherein $\mat{w}\in\mathbb{R}^{M}$ are the model weights and $\mat{\Lambda}$ is the associated prior covariance.
Once more considering a Gaussian likelihood (\cref{eq:likelihood}) yields a closed-form posterior \gls*{gp} whose mean and covariance require only the inversion of the matrix ${\sum_{n=1}^N\mat{\varphi}(\mat{x}_n){\mat{\varphi}(\mat{x}_n)}\trans}$. This yields a computational complexity of $\mathcal{O}(NM^2+M^3)$, which allows to tackle large-sampled data when $N \gg M$.

\subsection{Product Kernels}
In the remainder of this paper we consider \glspl*{gp} with product kernels.
\begin{definition}[Product kernel \citep{rasmussen_gaussian_2006}]\label{def:product_kernel}
A kernel $k(\mat{x},\mat{x'})$ is a product kernel if
    \begin{equation}\label{eq:product_kernel}
    k(\mat{x},\mat{x'}) = \prod_{d=1}^D k^{(d)} (x_d,x_d'),
\end{equation}
where each $k^{(d)} (\cdot,\cdot) : \mathbb{R}\times\mathbb{R}\rightarrow \mathbb{R}$ is a valid kernel.
\end{definition}
While many commonly used kernels are product kernels e.g. the Gaussian kernel and the polynomials kernel, product kernels provide a straightforward strategy to extend one-dimensional kernels to the higher-dimensional case \citep{rasmussen_gaussian_2006,hensman_variational_2017}.
The basis functions and prior covariance of product kernels can then be determined based on the basis function expansion of their constituents as follows.
\begin{lemma}[Basis functions and prior covariances of product kernels]\label{lemma:product_kernels}
    Consider the product kernel of \cref{def:product_kernel}. Denote the basis functions and prior covariance of each factor $k^{(d)} (x_d,x_d')$ as ${\mat{\varphi}^{(d)}(x_d)\in\mathbb{R}^{M_d}}$ and $\mat{\Lambda}^{(d)}\in\mathbb{R}^{M_d\times M_d}$ respectively, then the basis functions and prior covariance of $k(\mat{x},\mat{x'})$ are
    \begin{equation}\label{eq:tensor-product-basis}
    \mat{\varphi}(\mat{x}) = \otimes_{d=1}^D \mat{\varphi}^{(d)}(x_d),
    \end{equation}
    and
    \begin{equation}\label{eq:tensor-product-prior}
        \mat{\Lambda} = \otimes_{d=1}^D \mat{\Lambda}^{(d)},
    \end{equation}
\end{lemma}

The inherent challenge in this approach stems from the exponential increase of the number of basis functions $M$ and thus of model parameters as a function of the dimensionality of the input data, thereby restricting their utility to low-dimensional datasets.

Such structure arises for instance when dealing with Mercer expansions of product kernels, in the structured kernel interpolation framework \citep{wilson_kernel_2015,yadav_faster_2021} variational Fourier features framework \citep{hensman_variational_2017} and Hilbert-GP framework \citep{solin_hilbert_2020}.
Alternative important approximation strategies which avoid this exponential scaling are random features \citep{rahimi_random_2007,lazaro-gredilla_sparse_2010}, inducing features \citep{csato_sparse_2002,seeger_fast_2003,quinonero-candela_unifying_2005,snelson_sparse_2006,hensman_gaussian_2013,hensman_scalable_2015} and additive \glspl*{gp} \citep{duvenaud_additive_2011,lu_additive_2022} which circumvent the outlined computational issue.
All those approaches can be interpreted as projecting the \gls*{gp} on a set of \glspl*{bf}.

The performance of these methods however tends to deteriorate in higher dimensions, as they need to cover an exponentially large domain with a linear number of random samples or inducing points. These issues are some of the computational aspects of the \emph{curse of dimensionality}, which renders it difficult to operate in high-dimensional feature spaces \citep{hastie_elements_2001}.
\subsection{Tensor Networks}
A recent alternative approach to remedy said curse of dimensionality affecting the exponentially increasing weights of the linear model in \cref{eq:low_rank_GP} consists in constraining the models weights $\mat{w}$ to be a low-rank tensor network. 
\glspl*{tn} express a $D$-dimensional tensor $\ten{W}$ as a multi-linear function of a number of \emph{core} tensors.
Two commonly used  \glspl*{tn}  are the \gls*{cpd} and \gls*{tt}, defined as follows.
\begin{definition}[\gls*{cpd} \citep{hitchcock_expression_1927}]\label{def:CPD}
A $D$-dimensional tensor $\ten{W}\in\mathbb{R}^{M_1\times M_2 \times \cdots \times M_D}$ has a rank-$R$ \gls*{cpd} if
\begin{equation}\label{eq:CPD}
    w_{m_1,m_2,\dots, m_D}= \sum_{r=1}^{R} \prod_{d=1}^D {w^{(d)}}_{m_d, r}.
\end{equation}
The cores of a \gls*{cpd} are the matrices $\mat{W}^{(d)}\in\mathbb{R}^{M_d \times R}$.
Since a \gls*{cpd} tensor can be expressed solely in terms of its cores, its storage requires $P_{\ff{CPD}} = R \sum_{d=1}^D M_d$ parameters as opposed to $\prod_{d=1}^D M_d$.
\end{definition}
\begin{definition}[\gls*{tt} \citep{oseledets_tensor-train_2011}]\label{def:TT}
A $D$-dimensional tensor $\ten{W}\in\mathbb{R}^{M_1 \times M_2 \times  \cdots \times  M_D}$ admits a rank-$(R_0\isdef 1,R_1,\ldots,R_D\isdef 1)$ tensor train if
\begin{equation}
      w_{m_1,m_2,\ldots, m_D} =  \sum_{r_0=1}^{R_0} \sum_{r_1=1}^{R_1} \cdots \sum_{r_{D}=1}^{R_{D}} \prod_{d=1}^D {w^{(d)}}_{r_{d-1}, m_d, r_{d}}.
\end{equation}
The cores of a tensor train are $D$ $3$-dimensional tensors $\ten{W}^{(d)}\in\mathbb{R}^{R_{d-1}\times M \times R_{d}}$ which yield ${P_{\ff{TT}} = \sum_{d=1}^D  R_{D-1} M_D R_{D}}$ parameters. 
\end{definition}

In the following we denote by $\TN{\ten{W}}$ a tensor which admits a general \gls*{tn} format, by $\CPD{\ten{W}}$ a tensor which admits a rank-$R$ CP form and by $\TT{\ten{W}}$ a tensor in rank-$(R_0\isdef 1,R_1,\ldots,R_D\isdef 1)$ \gls*{tt} form. Lastly, we denote by $\RANKONE{\ten{W}}$ a tensor which is in rank-$1$ CP form or rank-$(1,1,\ldots,1)$ \gls*{tt}, as both are equivalent.

Importantly, we refer to a tensor in general \gls*{tn} format $\TN{\ten{W}} \in \mathbb{R}^{M_1\times M_2 \times \cdot \times M_D}$ as \emph{underparametrized} if its rank hyperparameters, e.g. $R$ in case of \gls*{cpd}, are chosen such that its storage cost is less than $\prod_{d=1}^D M_d$.  This is crucial in order to obtain storage, and as we will see, computational benefits.

\subsection{Tensor Network-Constrained Kernel Machines}\label{sub:tnkm}
 \glspl*{tn}  have been used to reduce the number of model parameters in kernel machines (\cref{eq:low_rank_GP}) by tensorizing the \glspl*{bf} $\mat{\varphi}(\mat{x})$ and model weights $\mat{w}$ and by constraining both to be underparameterized \glspl*{tn}.
This approach lays its foundations on the fact that the Frobenius inner product of a tensorized vector is isometric with respect to the Euclidean inner product, i.e.
\begin{equation}
    f(\mat{x}) = \inner{\mat{\varphi}(\mat{x})}{\mat{w}} = \innerfrob{\tensorize{\mat{\mat{\varphi}(\mat{x})}}}{\tensorize{\mat{w}}}.
\end{equation}
This isometry allows then to constrain the \glspl*{bf} and the model weights to be an underparameterized \gls*{tn}. 
Since product kernels yield an expansion in terms of Kronecker-product \glspl*{bf} (\cref{eq:tensor-product-basis}), they are a rank-$1$ \gls*{tn} by definition after tensorization. Embedding these relations yields an approximate model
\begin{equation}\label{eq:TN_MODEL}
     f(\mat{x}) \approx f_{\ff{TN}}(\mat{x}) \isdef \innerfrob{\RANKONE{\tensorize{\mat{\varphi}(\mat{x})}}}{\TN{\tensorize{\mat{w}}}},
\end{equation}
characterized by lower storage and computational complexities.
This approach has been proposed mostly for weights modeled as \gls*{cpd} \citep{kargas_supervised_2021,wesel_large-scale_2021,wesel_tensor-based_2023} or \gls*{tt} \citep{wahls_learning_2014,stoudenmire_supervised_2016,batselier_tensor_2017-1,novikov_exponential_2018,chen_parallelized_2018} as they arguably introduce fewer hyperparameters (only one in case of \gls*{cpd}) and thus are in practice easier to work with compared to other \glspl*{tn} such as the \gls*{mera} \citep{reyes_multi-scale_2021}.

We define such models as we will need them in detail in the next section, where we present our main contribution.

\begin{definition}[\gls*{cpd}-constrained kernel machine]\label{def:cpdkm}
     The \gls*{cpd}-constrained kernel machine is defined as
    \begin{align}
        f_{\ff{CPD}}(\mat{x}) &\isdef \innerfrob{\RANKONE{\tensorize{\mat{\varphi}(\mat{x})}}}{\CPD{\tensorize{\mat{w}}}} \label{eq:f_CPD} \\
        &= \sum_{r=1}^R  \hiddenCP_r(\mat{x}),
    \end{align}
    where the intermediate variables $h_r\in\mathbb{R}$ are defined as
    \begin{equation}
        \hiddenCP_r(\mat{x}) \isdef \prod_{d=1}^D {\mat{\varphi}^{(d)}(x_d)} \trans{\mat{w}^{(d)}}_{:,r} \label{eq:addend_CPD}. 
    \end{equation}
\end{definition}
Similarly, we provide a definition for the \gls*{tt}-constrained kernel machine.
\begin{definition}[\gls*{tt}-constrained kernel machine]\label{def:ttkm}
    The \gls*{tt}-constrained kernel machine is defined as
    \begin{align}
        f_{\ff{TT}}(\mat{x}) &\isdef \innerfrob{\RANKONE{\tensorize{\mat{\varphi}(\mat{x})}}}{\TT{\tensorize{\mat{w}}}} \label{eq:f_TT} \\
        &=\sum_{r_{D}=1}^{R_{D}} \sum_{r_{D-1}=1}^{R_{D-1}} \cdots \sum_{r_0=1}^{R_0}\prod_{d=1}^D z^{(d)}_{r_{d-1}, r_{d}} (x_d),
    \end{align}
    where the intermediate variables $\tempTTMat{d}\in\mathbb{R}^{R_{d-1} \times R_d}$ are defined element-wise as
    \begin{equation}
         z^{(d)}_{r_{d-1},r_d} (x_d) \isdef \sum_{m_d = 1}^{M_d} \varphi^{(d)}_{m_d}(x_d) w^{(d)}_{r_{d-1}, m_d , r_{d}}.
    \end{equation}
\end{definition}
Evaluating \gls*{cpd} and \gls*{tt}-constrained kernel machines (\cref{eq:f_CPD}, \cref{eq:f_TT}) and their gradients can be accomplished with \order{P_{\ff{CPD}}} and \order{P_{\ff{TT}}} computations, respectively. This allows the practitioner to tune the rank hyperparameter in order to achieve a model that fits in the computational budget at hand and that learns from the specified \glspl*{bf}.

From an optimization point-of-view, models in the form of \cref{eq:TN_MODEL} have been trained both in the \gls*{ml} \citep{stoudenmire_supervised_2016,batselier_tensor_2017-1} and in the \gls*{map} setting \citep{wahls_learning_2014,novikov_exponential_2018,chen_parallelized_2018,kargas_supervised_2021,wesel_large-scale_2021,wesel_tensor-based_2023} and in the context of \gls*{gp} variational inference \citep{izmailov_scalable_2018} where \glspl*{tt} are used to parameterize the variational distribution. 
It is however not clear if and how these models relate to the weight-space \gls*{gp} \cref{eq:low_rank_GP}.

In the following section we present the main contribution of our work: we show how when placing i.i.d. priors on the cores of these approximate models, they converge to a \gls*{gp} which we fully characterize.

\section{\gls*{tn}-Constrained Kernel Machines as \glspl*{gp}}\label{sec:main}

\begin{figure*}
    \centering
    \subfigure[$P_{\ff{CPD}}=\num{1600}$]{\includegraphics[width=0.25\linewidth]{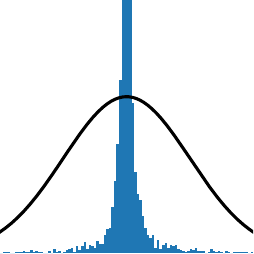}\label{fig:hist_cpd_1}}\hfill
    \subfigure[$P_{\ff{CPD}}=\num{1600000}$]{\includegraphics[width=0.25\linewidth]{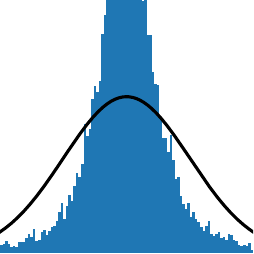}\label{fig:hist_tt_1}}\hfill
    \subfigure[$P_{\ff{TT}}=\num{1320}$]{\includegraphics[width=0.25\linewidth]{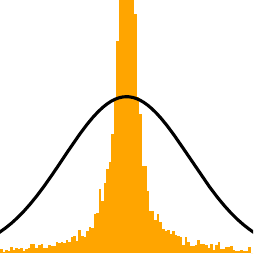}\label{fig:hist_cpd_2}}\hfill
    \subfigure[$P_{\ff{TT}}=\num{1605000}$]{\includegraphics[width=0.25\linewidth]{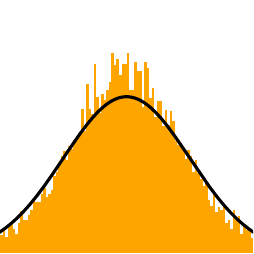}\label{fig:hist_tt_2}}\hfill
    \caption{Histograms of the empirical \gls*{pdf} of \gls*{cpd} (blue) and \gls*{tt} (orange) models specified in \cref{thm:CPD,thm:TT} evaluated at a random point as a function of model parameters $P$ for $D=16$.
    The black line is the \gls*{pdf} of the \gls*{gp}.
    Notice how \gls*{tt} converges faster to the \gls*{gp} for the same number of model parameters $P$.}
    \label{fig:histograms}
\end{figure*}

We commence to outline the correspondence between \gls*{tn}-constrained kernel machine and \glspl*{gp}, which makes use of the \gls*{clt}. We begin by elucidating the simplest case, i.e. the \gls*{cpd}.

\begin{theorem}[\gls*{gp} limit of \gls*{cpd}-constrained kernel machine]\label{thm:CPD}
    Consider the \gls*{cpd}-constrained kernel machine
    \begin{equation*}
        f_{\ff{CPD}}(\mat{x}) \isdef \innerfrob{\RANKONE{\tensorize{\mat{\varphi}(\mat{x})}}}{\CPD{\tensorize{\mat{w}}}}.
    \end{equation*}
    If each of the $R$ columns ${\mat{w}^{(d)}}_{:,r}\in\mathbb{R}^{M_d}$ of each \gls*{cpd} core is an i.i.d. random variable such that
    \begin{align*}\label{eq:CPD_prior}
        \E \left[\mat{w}^{(d)}_{:,r}\right] &= 0, \\
        \E\left[\mat{w}^{(d)}_{:,r} {\mat{w}^{(d)}_{:,r}}\trans\right] &= {R^{-\frac{1}{D}}} \mat{\Lambda}^{(d)},
    \end{align*}
    then $f_{\ff{CPD}}(\mat{x})$ converges in distribution as $R\rightarrow\infty$ to the \gls*{gp} 
    \begin{equation*}
        f_{\ff{CPD}}(\mat{x}) \sim \mathcal{GP}\left(0, \prod_{d=1}^D {\mat{\varphi}^{(d)}(x_d)}\trans\mat{\Lambda}^{(d)} \mat{\varphi}^{(d)}(\cdot)\right).
    \end{equation*}
\end{theorem}
\begin{proof}
    See \cref{proof:cpd}.
\end{proof}
A similar result can be constructed for the \gls*{tt} case.
\begin{theorem}[\gls*{gp} limit of \gls*{tt}-constrained kernel machine]\label{thm:TT}
    Consider the \gls*{tt}-constrained kernel machine
    \begin{equation*}
        f_{\ff{TT}}(\mat{x}) \isdef \innerfrob{\RANKONE{\tensorize{\mat{\varphi}(\mat{x})}}}{\TT{\tensorize{\mat{w}}}}
    \end{equation*}
    If each of the $R_{d-1} R_d$ fibers ${\mat{W}^{(d)}}_{r_{d-1},:,r_{d}}\in\mathbb{R}^{M_d}$ of each \gls*{tt} core is an i.i.d. random variable such that
    \begin{align*}\label{eq:TT_prior}
        \E \left[\mat{W}^{(d)}_{r_{d-1},:,r_{d}}\right] &= \mat{0}, \\
        \E\left[\mat{W}^{(d)}_{r_{d-1},:,r_{d}} {\mat{W}^{(d)}_{r_{d-1},:,r_{d}}}\trans\right] &= \frac{1}{\sqrt{R_{d-1} R_{d}}} \mat{\Lambda}^{(d)},
    \end{align*}
    then $f_{\ff{TT}}(\mat{x})$ converges in distribution as $R_1\rightarrow\infty$, $R_2\rightarrow\infty$, \ldots, $R_{D-1}\rightarrow\infty$ to the Gaussian process 
    \begin{equation*}
        f_{\ff{TT}}(\mat{x}) \sim \mathcal{GP}\left(0, \prod_{d=1}^D {\mat{\varphi}^{(d)}(x_d)}\trans\mat{\Lambda}^{(d)} \mat{\varphi}^{(d)}(\cdot)\right).
    \end{equation*}
\end{theorem}
\begin{proof}
    See \cref{proof:1}.
\end{proof}
\Cref{thm:TT} guarantees the convergence in distribution of $f_{\ff{TT}}(\mat{x})$ to the \gls*{gp} of \cref{eq:low_rank_GP} by taking successive limits of each \gls*{tt} rank. Importantly, the same convergence results also holds true if the \gls*{tt} ranks grow simultaneously, see \cref{appendix:thm_mathews}.

Both \cref{thm:CPD,thm:TT} are remarkable, as they imply that a degenerate \gls*{gp} which can be defined with an exponential number of $\prod_{d=1}^D M_d$ weights $\mat{w}$ can be also obtained with an \emph{infinite} number of model parameters $P$ using the \gls*{cpd}-constrained model of \cref{def:cpdkm} or the \gls*{tt}-constrained model of \cref{def:ttkm}. 
Further, \cref{thm:CPD,thm:TT} suggest that \gls*{cpd} and \gls*{tt}-based models exhibit \gls*{gp} behavior in the overparameterized regime.
\gls*{gp} behavior is characterized by a fixed learning representation, which in case of the kernel in \cref{thm:CPD,thm:TT} is fully defined by the \glspl*{bf} and is hence data-independent. 
On the contrary, in the finite rank regime, both \gls*{cpd} and \gls*{tt} models are able to craft nonlinear features from the provided \glspl*{bf}, potentially learning more latent patterns in the data.

\subsection{Convergence Rates to the \gls*{gp}}\label{subsec:convergence_rates}

While both \cref{thm:CPD} and \cref{thm:TT} guarantee convergence in distribution to the \gls*{gp} of \cref{eq:low_rank_GP}, they do so at rates that differ in terms of the number of model parameters. Let us assume, for simplicity, that the number of basis functions is the same along each dimension, i.e., $M$, and that the $D-1$ \gls*{tt} ranks equal $R$. 
It follows then that the number of \gls*{cpd} model parameters $P_{\ff{CPD}} = MDR_{\ff{CPD}}$ and the number of \gls*{tt} model parameters $P_{\ff{TT}} = M(D-2)R_{\ff{TT}}^2 + 2MR_{\ff{TT}} =\mathcal{O}(MDR_{\ff{TT}}^2)$.
Given the convergence rate of the \gls*{clt} for the expression in \cref{eq:f_CPD} to the \gls{gp} in \cref{eq:low_rank_GP}, denoted as $\mathcal{O}(\nicefrac{1}{\sqrt{R_{\ff{CPD}}}})$ with respect to the variable $P_{\ff{CPD}}$, we can establish the following corollary by substituting $R_{\ff{CPD}}$ as a function of $P_{\ff{CPD}}$.
\begin{corollary}[Convergence rate for \gls*{cpd}]\label{cor:cpd}
    Under the conditions of \cref{thm:CPD}, the function $f_{\ff{CPD}}(\mat{x})$ converges in distribution to the \gls*{gp} defined by \cref{eq:low_rank_GP}. The convergence rate is given by:
    \begin{equation*}
        f_{\ff{CPD}}(\mat{x}) \rightarrow \mathcal{O}\left(\left({\frac{MD}{P_{\ff{CPD}}}}\right)^\frac{1}{2}\right).
    \end{equation*}
\end{corollary}
Due to their hierarchical structure, \gls*{tt} models are a composition of $R_{\ff{TT}}^{D-1}$ variables, but can be represented in a quadratic number of model parameters in $R_{\ff{TT}}$, since $P_{\ff{TT}}=\mathcal{O}(MDR_{\ff{TT}}^2)$. Expressing then the \gls*{clt} convergence rate of $\mathcal{O}(\nicefrac{1}{\sqrt{R_{\ff{TT}}}^{D-1}})$ as a function of $P_{\ff{TT}}$ yields the following corollary.
\begin{corollary}[Convergence rate for \gls*{tt}]\label{cor:tt}
    Under the conditions of \cref{thm:TT}, the function $f_{\ff{TT}}(\mat{x})$ converges in distribution to the \gls*{gp} defined by \cref{eq:low_rank_GP}. The convergence rate is given by:
\begin{equation*}
    f_{\ff{TT}}(\mat{x}) \rightarrow \mathcal{O}\left(\left({\frac{MD}{P_{\ff{TT}}}}\right)^\frac{D-1}{4}\right).
\end{equation*}
\end{corollary}
Therefore, when dealing with identical models in terms of the number of basis functions ($M$), dimensionality of the inputs ($D$), and the number of model parameters ($P_{\ff{CPD}}=P_{\ff{TT}}$), $f_{\ff{TT}}(\mat{x})$ will converge at a polynomially faster rate than $f_{\ff{CPD}}(\mat{x})$, thus exhibiting \gls*{gp} behavior with a reduced number of model parameters.
In particular, based on \cref{cor:cpd,cor:tt} we expect the \gls*{gp} convergence rate of \gls*{tt} models to be faster for $D\geq3$.

These insights are relevant for practitioners engaged with \gls*{tn}-constrained kernel machines, as they shed light on the balance between the \gls*{gp} and (deep) neural network behavior inherent in these models. Notably, \gls*{cpd} and \gls*{tt}-constrained models, akin to shallow and \glspl*{dnn} respectively, have the capacity to craft additional nonlinearities beyond the provided basis functions.
This characteristic can result in superior generalization when dealing with a limited number of parameters. However, as the parameter count increases, we expect these models to transition towards \gls*{gp} behavior, characterized by a fixed feature representation and static in comparison.

\subsection{Consequences for \gls*{map} Estimation}
As discussed in \cref{sub:tnkm}, \gls*{tn}-constrained kernel machines are typically trained in the \gls*{ml} or \gls*{map} framework by constraining the weights $\mat{w}$ in the log-likelihood or log-posterior to be a \gls*{tn}. 
In said \gls*{map} context, and e.g. when specifying a normal prior on the model weights $\mat{w}\sim \mathcal{N}(\mat{0},\mat{\Lambda})$, the resulting regularization term $\Omega$ is approximated by  $\Omega_{\ff{TN}}$ as
\begin{equation*}
\Omega \isdef\norm{\mat{\Lambda}^{-\frac{1}{2}}\mat{w}}^2 \approx \Omega_{\ff{TN}} \isdef \norm{\TN{\tensorize{\mat{\Lambda}^{-\frac{1}{2}} \mat{w}}}}^2,
\end{equation*}
where $\mat{\Lambda} = \kron_{d=1}^D \mat{\Lambda}^{(d)}$.
For example, in case of \gls*{cpd}-constrained models we have
\begin{equation}\label{eq:reg_full_CPD}
    \Omega_{\ff{CPD}} =  \norm{\odot_{d=1}^D \left({\mat{W}^{(d)}}\trans {\mat{\Lambda}^{(d)}}^{-1} \mat{W}^{(d)}\right)}^2.
\end{equation}
This form of regularization is considered for \gls*{tt} by \citet{wahls_learning_2014,novikov_exponential_2018,chen_parallelized_2018} and for \gls*{cpd} by \citet{wesel_large-scale_2021,wesel_tensor-based_2023}. It provides a Frobenius norm approximation of the regularization term which recovers the original \gls*{map} estimate as the hyperparameters of $\TN{\tensorize{\Lambda \mat{w}}}$ are chosen such that $\TN{\tensorize{\mat{\Lambda}\mat{w}}}=\tensorize{\mat{\Lambda w}}$.
If we now consider the regularization term in the log-posterior of \cref{thm:CPD} we end up with
\begin{equation}\label{eq:reg_new_CPD}
     \Omega_{\ff{CPD}} \isdef R^{\frac{1}{D}} \sum_{d=1}^D \norm{{\mat{{\Lambda}}^{(d)}}^{-\frac{1}{2}} \mat{W}^{(d)}}^2.
\end{equation}
This regularization has been applied without the scaling factor $ R^{\frac{1}{D}}$ and with $\mat{\Lambda}^{(d)} = \eye{M_d}$, as observed in the work of \citet{kargas_supervised_2021}, who may not have been aware of the underlying connection at that time.
It does not account for the model interactions across the dimensionality due to the i.d.d. assumptions on the cores.
Contrary to the regularization $\Omega_{\ff{TN}}$ in \cref{eq:reg_full_CPD}, it provides an approximation which recovers the log-prior $\Omega$ and thus the \gls*{map} \emph{in the limit}. 
These considerations point to the fact that if the practitioner is interested only in a \gls*{map} estimate which gives weight to the full prior $\mat{w}\sim \mathcal{N}(\mat{0},\mat{\Lambda})$ as well as possible, he might be more interested in the regularization of \cref{eq:reg_full_CPD}.
Furthermore, the priors in \cref{thm:CPD,thm:TT} provide a sensible initial guess for gradient-based optimization which is invariant w.r.t. the dimensionality of the inputs and the choice of rank hyperparameters. We hence directly address the initialization issues affecting \gls*{tn}-constrained kernel machines \citep{barratt_improvements_2021} by providing a sensible initialization strategy and regularization which does not suffer from vanishing or exploding gradients.

\section{Numerical Experiments}\label{sec:experiments}
\begin{figure*}[t]
    \centering
    \subfigure{\includegraphics[width=0.25\linewidth]{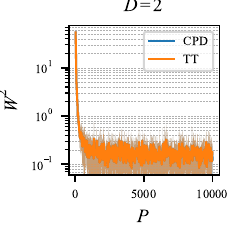}\label{sub:vm1}}\hfill
    \subfigure{\includegraphics[width=0.25\linewidth]{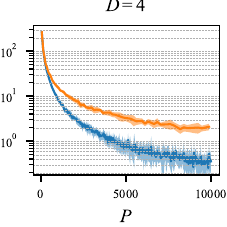}\label{sub:vm2}}\hfill
    \subfigure{\includegraphics[width=0.25\linewidth]{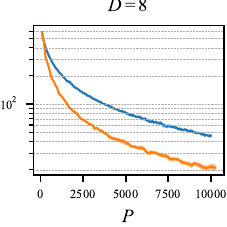}\label{sub:vm3}}\hfill
    \subfigure{\includegraphics[width=0.25\linewidth]{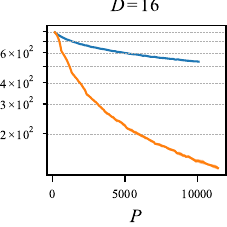}\label{sub:vm4}}\hfill
    \caption{Mean and standard deviation of the Cramér–von Mises statistic $W^2$ evaluated between the empirical \gls*{cdf} of \gls*{cpd} and \gls*{tt} models specified in \cref{thm:CPD,thm:TT} evaluated at $N=10$ random points as a function of model parameters $P$ for $D=2,4,8,16$.
    The two models are equivalent for $D=2$. Notice how \gls*{tt} converges faster to the \gls*{gp} as the dimensionality of the inputs $D$ increases.}
    \label{fig:cramer_von_mises}
\end{figure*}
We setup two numerical experiments in order to respectively empirically observe the claims in \cref{thm:CPD,thm:TT} by evaluating the convergence to the prior \gls*{gp} in \cref{eq:low_rank_GP}, and to evaluate the \gls*{gp} behavior of such models at prediction in the finite rank case.
In all experiments we made use of the Hilbert-GP \cite{solin_hilbert_2020} to provide a \gls*{bf} expansion for the Gaussian kernel and opt for $M=10$ basis functions per dimension. 
When sampling the cores of the models in \cref{thm:CPD,thm:TT} we resorted to normally distributed random variables. We implemented both models in PyMC \citep{abril-pla_pymc_2023}. The fully anonymized Python implementation is available at \url{github.com/fwesel/tensorGP}.

\subsection{\gls*{gp} Convergence}
In order to empirically verify the convergence to the \gls*{gp} of \cref{eq:low_rank_GP} we sample $\num{10000}$ instances of the \gls*{cpd} and \gls*{tt} models specified in \cref{thm:CPD,thm:TT} for increasing \gls*{cpd} and \gls*{tt} ranks yielding up to $P=\num{10000}$ model parameters.
Since the target distribution is Gaussian with known moments, we record the Cramér–von Mises statistic $W^2$ \citep{dagostino_goodness--fit_1986} which gives a metric of closeness between the target and our sampled empirical \gls*{cdf}.
We repeat this for $N=10$ randomly sampled data points and for $D=2,4,8,16$ and report the mean and standard deviation of the results in \cref{fig:cramer_von_mises}. 
Therein it can be observed that for the same number of model parameters, \gls*{tt} converges more rapidly than \gls*{cpd} as the dimensionality of the inputs grows.
Both approaches however need exponentially more parameters to converge at the same rate for increasing dimensionality of the inputs.
Note that for $D=4$ \gls*{cpd}, contrary to what stated in \cref{subsec:convergence_rates} still converges faster due to the approximation made when considering $P_{\ff{TT}}=DMR^2$.
Histograms of the empirical \gls*{cdf} for one datapoint are shown in \cref{fig:histograms}.
This behavior stems from the fact that for a fixed combination of $D$, $M$ and $P$, \gls*{tt} captures an exponential $R^{D-1}$ range of model interactions in contrast to the $R$ linear interactions exhibited by \gls*{cpd}. This fact renders the choice of \gls*{tt} more suitable with respect to \gls*{cpd} if one wishes the model to exhibit more \gls*{gp} behavior, which is characterized by a weight-independent feature representation and is less likely to overfit.

\subsection{\gls*{gp} Behavior at Prediction}
To investigate whether \gls*{cpd} and \gls*{tt}-constrained kernel machines exhibit \gls*{gp} behavior as the number of parameters increases we tackle two small \emph{UCI} regression problems, \emph{yacht} and \emph{energy} \citep{dua_uci_2017}.
We consider $\SI{70}{\percent}$ of the data for training an the remaining for test and model our observations as having i.i.d. Gaussian likelihood (\cref{eq:likelihood}).
We then consider the \gls*{gp} in \cref{thm:CPD,thm:TT} which we train by maximizing the marginalized likelihood (\num{10} random initializations). In order to compare models with the target \gls*{gp}, we fix the obtained hyperparameters $\{\sigma,\mat{\Lambda}\}$ and sample $\num{4}$ chains of $\num{2000}$ instances from the posteriors $p(\CPD{\tensorize{\mat{w}}}\mid \mat{y})$ and $p(\TT{\tensorize{\mat{w}}}\mid \mat{y})$ for a range of model parameters $P$ using the \emph{No U-Turn Sampling} Hamiltonian Monte Carlo scheme with default parameters. 
After discarding the first $\num{1000}$ samples (burn-in), we obtain posterior predictive distributions $p(f_{\ff{CPD}}(\mat{x})\mid\mat{y})$ and $p(f_{\ff{TT}}(\mat{x})\mid\mat{y})$.
We then compare the two models in terms of the \gls*{rmse} of the posterior mean on the test data. We plot the mean and standard deviation of the \gls*{rmse} over the chains in \cref{fig:posterior_mean}.

In \cref{fig:posterior_mean} one can observe that the prediction of both \gls*{cpd} and \gls*{tt} models tends towards the \gls{gp} as the number of model parameters increases. 
In case of the \emph{yacht} dataset this happens with smaller errors, i.e. both models generalize better.
As expected, the \gls*{tt}-constrained kernel machine exhibits more \gls*{gp} behavior and in this case worse generalization.
In case of the \emph{energy} dataset both methods have very similar performance compared to the \gls*{gp} and converge with larger test errors as soon as their ranks are different than one.
While the behavior on both datasets can be explained in term of the \gls*{gp} underfitting on the \emph{yacht} dataset, it is worth noting that both \gls*{tt} and \gls*{cpd}-based models perform at least as well as the \gls*{gp} for ranks higher than one, rendering them computationally advantageous alternatives to the \gls*{gp} in this context.

\begin{figure*}
    \centering
    \subfigure{\includegraphics[width=0.5\linewidth]{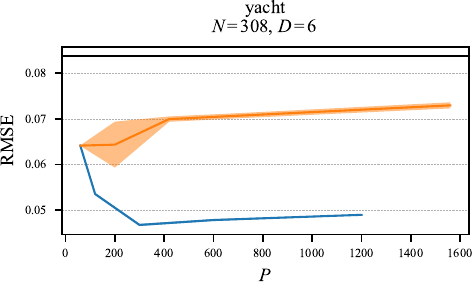}\label{sub:yacht}}\hfill
    \subfigure{\includegraphics[width=0.5\linewidth]{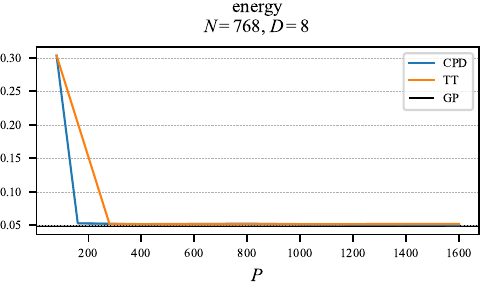}\label{sub:energy}}\hfill
    \caption{Mean and standard deviation of the test \gls*{rmse} evaluated of \gls*{cpd} and \gls*{tt} models as a function of model parameters $P$ as well as their target \gls*{gp}.
    Notice how in the \emph{yacht} dataset \gls*{tt} exhibits more \gls*{gp} behavior compared to \gls*{cpd} as $P$ increases. On the \emph{energy} dataset both methods exhibit \gls*{gp} behavior already for ranks different than one, which explains why \gls*{tt} appears to be slower.}
    \label{fig:posterior_mean}
\end{figure*}

\section{Related Work}\label{sec:related_work}

Our contribution is closely tied to the links between Bayesian neural networks and \glspl*{gp}, first established for single-layer single-output neural networks \citep{neal_bayesian_1996,neal_priors_1996} having sigmoidal \citep{williams_computing_1996}, Gaussian \citep{williams_computing_1997} and rectified linear unit \citep{cho_kernel_2009} as activation function.

This idea was extended to \glspl*{dnn} by \citet{lee_deep_2018} and \citet{matthews_gaussian_2018} for the all the most common activation functions. It is important to note that the derivation of \citet{lee_deep_2018} makes use of recursive application of the \gls*{clt} to each layer of the network, while \citet{matthews_gaussian_2018} considers the case where the width of each layer grows simultaneously to the others, which is arguably more valuable in practice.

Further extensions have been proposed to \glspl*{cnn} where the number of channels tends to infinity \citep{novak_bayesian_2018,garriga-alonso_deep_2018}, to \glspl*{rnn} \citep{sun_recurrent_2022} and to \glspl*{dnn} having low-rank constraints on the weight matrices \citep{nait-saada_beyond_2023}.

In particular \cref{thm:CPD} resembles the results of \citet{neal_bayesian_1996,neal_priors_1996,williams_computing_1997} which relate infinite-width single layer neural networks to \glspl*{gp}. The \gls*{cpd} rank corresponds exactly to the width of the neural network. The crucial difference lies however in the product structure, which is not present in neural networks and introduces a nonlinearity of different kind than the one induced by the activation function, which is linear in the \gls*{cpd} and more in general for any \gls*{tn}. \glspl*{tt} on the other hand resemble \glspl*{dnn} as they map the output of each core to the next one. However, in contrast to \glspl*{dnn}, the inputs are processed over the depth of the network. For a more in depth discussion we refer the reader to \citep{cohen_expressive_2016}.

Likewise \cref{thm:TT} is the \gls*{tn} counterpart to the works of \citet{lee_deep_2018,matthews_gaussian_2018} which relate finite depth neural networks to \glspl*{gp}. The \gls*{tt} ranks are then equivalent to the width of each layer and the activation function is linear.
In contrast to the \glspl*{dnn} case, the induced \glspl*{gp} are degenerate.

\section{Conclusion}\label{sec:conclusion}
In this paper we proved that \gls*{cpd} and \gls*{tt}-constrained kernel machines are \glspl*{gp} in the limit or large \gls*{tn} ranks.
We characterized the target \gls*{gp}, analyzed the convergence behavior of both models and showed that compared to \gls*{cpd}, \gls*{tt}-based models converge faster to the \gls*{gp} when dealing with higher-dimensional inputs. We empirically demonstrated these properties by means of numerical experiments. 

While the \gls*{gp} convergence is too slow to warrant these models as \gls*{gp} prior approximations, the insights we derived are useful in practice as they shed light on the effects of the choice of \gls*{tn} and regularization in these models.

\bibliography{Bibliography}
\bibliographystyle{abbrvnat}

\newpage
\appendix
\onecolumn
\section{Notation}\label{sec:notation}
Throughout this paper we denote scalars in both capital and non-capital italics $w,W$, vectors in non-capital bold $\mat{w}$, matrices in capital bold $\mat{W}$ and tensors in capital italic bold font $\ten{W}$. The $m$-th entry of a vector $\mat{w}\in\mathbb{R}^{M}$ is indicated as $w_m$ and the $m_1 m_2\ldots m_D$-th entry of a $D$-dimensional tensor $\ten{W}\in \mathbb{R}^{M_1 \times M_2 \times \cdots \times M_D}$ as $w_{m_1,m_2,\ldots m_D}$. 
We employ the column notation to indicate a set of elements of tensor given a set of indices, e.g. $\mat{W}_{m_1,:,m_2}$ and $\mat{W}_{m_1,1:3,m_2}$ represent respectively all elements and the first three elements along the second dimension of tensor $\ten{W}$ with fixed indices $m_1$ and $m_2$.
The Kronecker product is denoted by $\kron$ and the Hadamard (elementwise) by $\odot$. We employ one-based indexing for all tensors.
The Frobenius inner product between two $D$-dimensional tensors $\ten{V},\ten{W}\in\mathbb{R}^{M_1 \times M_2 \times  \cdots \times M_D}$ is
\begin{equation*}\label{eq:inner}
    \innerfrob{\ten{V}}{\ten{W}} \isdef \sum_{m_1=1}^{M_1}\sum_{m_2=1}^{M_2}\cdots \sum_{m_D=1}^{M_D} v_{m_1, m_2, \ldots, m_D} w_{m_1, m_2, \ldots, m_D},
\end{equation*}
and the Frobenius norm of $\ten{W}\in\mathbb{R}^{M_1 \times M_2 \times  \cdots \times M_D}$ is denoted and defined as
\begin{equation*}
    \norm{\ten{W}}^2 \isdef \innerfrob{\ten{W}}{\ten{W}}.
\end{equation*}
We define the vectorization operator as $\vectorize{\cdot}:\mathbb{R}^{M_1\times M_2 \times \cdots \times M_D} \rightarrow \mathbb{R}^{M_1 M_2 \cdots M_D}$ such that
\begin{equation*}
    \vectorize{\ten{W}}_{m} = w_{m_1,m_2,\ldots, m_D},
\end{equation*}
 with $m = m_1 + \sum_{d=2}^{D} (m_d-1) \prod_{k=1}^{d-1} M_k$. Likewise, its inverse, the tensorization operator $\tensorize{\cdot}: \mathbb{R}^{M_1 M_2 \cdots M_D}\rightarrow \mathbb{C}^{M_1\times M_2 \times \ldots M_D}$ is defined such that
\begin{equation*}
    \tensorize{\mat{w}}_{m_1, m_2, \cdots, m_D} = w_m.
\end{equation*}

\section{Proofs}
\subsection{\gls*{gp} of \gls*{cpd}-Constrained Kernel Machine}\label{appendix:cpd}
\begin{theorem}[\gls*{gp} limit of \gls*{cpd}-constrained kernel machine]
    Consider the \gls*{cpd}-constrained kernel machine
    \begin{equation*}
        f_{\ff{CPD}}(\mat{x}) \isdef \innerfrob{\RANKONE{\tensorize{\mat{\varphi}(\mat{x})}}}{\CPD{\tensorize{\mat{w}}}}.
    \end{equation*}
    If each of the $R$ columns of each \gls*{cpd} core ${\mat{w}^{(d)}}_{:,r}\in\mathbb{R}^{M_d}$ is an i.i.d. random variable such that
    \begin{align*}\label{eq:CPD_prior}
        \E \left[\mat{w}^{(d)}_{:,r}\right] &= 0, \\
        \E\left[\mat{w}^{(d)}_{:,r} {\mat{w}^{(d)}_{:,r}}\trans\right] &= \frac{1}{R^{\frac{1}{D}}} \mat{\Lambda}^{(d)},
    \end{align*}
    then $f_{\ff{CPD}}(\mat{x})$ converges in distribution as $R\rightarrow\infty$ to the Gaussian process 
    \begin{equation*}
        f_{\ff{CPD}}(\mat{x}) \sim \mathcal{GP}\left(0, \prod_{d=1}^D {\mat{\varphi}^{(d)}(x_d)}\trans\mat{\Lambda}^{(d)} \mat{\varphi}^{(d)}(\cdot)\right).
    \end{equation*}
\end{theorem}

\begin{proof}\label{proof:cpd}
Consider the $R$ intermediate functions $\hiddenCP_r$ of \cref{eq:addend_CPD} which constitute the \gls*{cpd}-constrained model of \cref{eq:f_CPD}. Due to the i.i.d. assumption on ${\mat{w}^{(d)}}_{:,r}$ each addend is the same function of i.i.d. random variables and thus is itself i.i.d.. The mean of each addend is
\begin{equation}
    \E\left[\hiddenCP_r(\mat{x})\right] = \E\left[\prod_{d=1}^D {\mat{\varphi}^{(d)}(x_d)}\trans{\mat{w}^{(d)}}_{:,r}\right] = 0,
\end{equation}
due to the i.i.d assumption and the linearity of expectation.
Its covariance is 
\begin{subequations}
    \begin{align}
        &\E\left[\hiddenCP_r(\mat{x}) \hiddenCP_r(\mat{x'})\right] \\
        =& \E\left[\prod_{d=1}^D {\mat{\varphi}^{(d)}(x_d)}\trans{\mat{w}^{(d)}}_{:,r}  \prod_{d=1}^D {\mat{\varphi}^{(d)}(x_d')}\trans{\mat{w}^{(d)}}_{:,r}\right] \label{eq:CPD_proof_1} \\
        =& \E\left[\prod_{d=1}^D {\mat{\varphi}^{(d)}(x_d)}\trans{\mat{w}^{(d)}}_{:,r} {{\mat{w}^{(d)}}_{:,r}}\trans {\mat{\varphi}^{(d)}(x_d')} \right] \label{eq:CPD_proof_2}\\
        =& \prod_{d=1}^D  {\mat{\varphi}^{(d)}(x_d)}\trans\E\left[{\mat{w}^{(d)}}_{:,r} {{\mat{w}^{(d)}}_{:,r}}\trans \right] {\mat{\varphi}^{(d)}(x_d')} \label{eq:CPD_proof_3} \\  
        =& \frac{1}{R} \prod_{d=1}^D {\mat{\varphi}^{(d)}(x_d)}\trans \mat{\Lambda}^{(d)} {\mat{\varphi}^{(d)}(x_d')}. \nonumber 
    \end{align}
\end{subequations}
Here the step from \cref{eq:CPD_proof_1} to \cref{eq:CPD_proof_2} exploits the fact that the transpose of a scalar is equal to itself, the step from \cref{eq:CPD_proof_2} to \cref{eq:CPD_proof_3} is due to the linearity of expectation.
As the variances of each intermediate function $\hiddenCP_r$ are appropriately scaled, by the multivariate central limit theorem $f_{\ff{CPD}}(\mat{x})$ converges in distribution to a multivariate normal distribution, which is fully specified by its first two moments
\begin{subequations}
    \begin{align*}
    \E\left[f_{\ff{CPD}}(\mat{x})\right] &= 0, \\
    \E\left[f_{\ff{CPD}}(\mat{x}) f_{\ff{CPD}}(\mat{x'})\right] &= \prod_{d=1}^D {\mat{\varphi}(x_d)}\trans\mat{\Lambda}^{(d)} \mat{\varphi}({x_d}').
    \end{align*}
\end{subequations}

Since any finite collection of $\{f_{\ff{CPD}}(\mat{x}),\ldots,f_{\ff{CPD}}(\mat{x'})\}$ will have a joint multivariate normal distribution with the aforementioned first two moments, we conclude that $f_{\ff{CPD}}(\mat{x})$ is the Gaussian process
\begin{equation*}
    f_{\ff{CPD}}(\mat{x}) \sim \mathcal{GP}\left(0, \prod_{d=1}^D {\mat{\varphi}^{(d)}(x_d)}\trans\mat{\Lambda}^{(d)} \mat{\varphi}^{(d)}(\cdot)\right).
\end{equation*}
\end{proof}

\subsection{\gls*{gp} of \gls*{tt}-Constrained Kernel Machine in the Sequential Limit of the \gls*{tt} Ranks}\label{appendix:thm}

\begin{theorem}[\gls*{gp} in the limit of \gls*{tt}-constrained kernel machine]
    Consider the \gls*{tt}-constrained kernel machine
    \begin{equation*}
        f_{\ff{TT}}(\mat{x}) \isdef \innerfrob{\RANKONE{\tensorize{\mat{\varphi}(\mat{x})}}}{\TT{\tensorize{\mat{w}}}}
    \end{equation*}
    If each of the $R_{d-1} R_d$ fibers of each \gls*{tt} core ${\mat{W}^{(d)}}_{r_{d-1},:,r_{d}}\in\mathbb{R}^{M_d}$ is an i.i.d. random variable such that
    \begin{align*}\label{eq:TT_prior}
        \E \left[\mat{W}^{(d)}_{r_{d-1},:,r_{d}}\right] &= \mat{0}, \\
        \E\left[\mat{W}^{(d)}_{r_{d-1},:,r_{d}} {\mat{W}^{(d)}_{r_{d-1},:,r_{d}}}\trans\right] &= \frac{1}{\sqrt{R_{d-1} R_{d}}} \mat{\Lambda}^{(d)},
    \end{align*}
    then $f_{\ff{TT}}(\mat{x})$ converges in distribution as $R_1\rightarrow\infty$, $R_2\rightarrow\infty$, \ldots, $R_{D-1}\rightarrow\infty$ to the Gaussian process 
    \begin{equation*}
        f_{\ff{TT}}(\mat{x}) \sim \mathcal{GP}\left(0, \prod_{d=1}^D {\mat{\varphi}^{(d)}(x_d)}\trans\mat{\Lambda}^{(d)} \mat{\varphi}^{(d)}(\cdot)\right).
    \end{equation*}
\end{theorem}

\begin{proof}\label{proof:1}
    Define the vector of intermediate function $\hiddenTTMat{d+1}\in\mathbb{R}^{R_{d+1}}$ recursively as
    \begin{equation*}\label{eq:recursion}
        \hiddenTT{d+1}_{r_{d+1}} \isdef \sum_{r_{d}=1}^{R_d} z^{(d+1)}_{r_{d}, r_{d+1}} (x_{d+1}) \hiddenTT{d}_{r_d} ,
    \end{equation*}
    with  $h^{(0)}\isdef1$. Note that the first two moments of intermediate variable $z^{(d+1)}_{r_{d}, r_{d+1}} (x_{d+1})$ are
    \begin{subequations}
        \begin{align*}
            &\E \left[z^{(d+1)}_{r_{d}, r_{d+1}} (x_{d+1})\right] = 0, \\
            &\E\left[z^{(d+1)}_{r_{d}, r_{d+1}} (x_{d+1}) z^{(d+1)}_{r_{d}, r_{d+1}} (x_{d+1}')\right] \\
            & = \frac{1}{\sqrt{R_{d}R_{d+1}}}{\mat{\varphi}^{(d)}(x_d)}\trans\mat{\Lambda}^{(d)} \mat{\varphi}^{(d)}(x_d')
        \end{align*}
    \end{subequations}
    We proceed by induction.
    For the induction step suppose that $\hiddenTT{d}_{r_d}$ is a \gls*{gp}, identical and independent for every $r_d$ such that
    \begin{equation*}
        \hiddenTT{d}_{r_d}\sim\mathcal{GP}\left(0, \frac{1}{\sqrt{R_d}} \prod_{p=1}^d {\mat{\varphi}^{(p)}(x_p)}\trans\mat{\Lambda}^{(p)} \mat{\varphi}^{(p)}(\cdot) \right).
    \end{equation*}
    The scalar $\hiddenTT{d+1}_{r_{d+1}}$ is the sum of $R_d$ i.i.d. terms having mean
    \begin{equation*}
        \E\left[\hiddenTT{d+1}_{r_{d+1}} \right] 
        = \E\left[   z^{(d+1)}_{r_{d}, r_{d+1}} (x_{d+1}) \hiddenTT{d}_{r_d}\right] = 0,
    \end{equation*}
    and covariance
    \begin{subequations}
        \begin{align*}
            & \E\left[\hiddenTT{d+1}_{r_{d+1}}  \hiddenTT{d+1}_{r_{d+1}} \right] \\
            = & \E\left[  z^{(d+1)}_{r_{d}, r_{d+1}} (x_{d+1})  \hiddenTT{d}_{r_d}   z^{(d+1)}_{r_{d}, r_{d+1}} (x_{d+1}') \hiddenTT{d}_{r_d}\right] \\
            = & \E\left[z^{(d+1)}_{r_{d}, r_{d+1}} (x_{d+1}) z^{(d+1)}_{r_{d}, r_{d+1}} (x_{d+1}')\right]  \E\left[\hiddenTT{d}_{r_d} \hiddenTT{d}_{r_d}\right]\\
            = & \frac{1}{\sqrt{R_{d+1}}} \prod_{p=1}^{d+1}  {\mat{\varphi}^{(p)}(x_p)}\trans\mat{\Lambda}^{(p)} \mat{\varphi}^{(p)}(x_p').
    \end{align*}
    \end{subequations}
    Since the assumptions of the \gls*{clt} are satisfied $\hiddenTT{d+1}_{r_{d+1}}$ converges in distribution to the normal distribution, fully specified by the above mentioned first two moments. Since any finite collection of $\{ \hiddenTT{d+1}_{r_{d+1}}(\mat{x}_{1:d+1}),\ldots, \hiddenTT{d+1}_{r_{d+1}}(\mat{x'}_{1:d+1})\}$ will have a joint multivariate normal distribution with the aforementioned first two moments, we conclude that $\hiddenTT{d+1}_{r_{d+1}}(\mat{x}_{1:d+1})$ is the \gls*{gp}
    \begin{equation*}
        \hiddenTT{d+1}_{r_{d+1}}\sim\mathcal{GP}\left(0, \frac{1}{\sqrt{R_{d+1}}} \prod_{p=1}^{d+1} {\mat{\varphi}^{(p)}(x_p)}\trans\mat{\Lambda}^{(p)} \mat{\varphi}^{(p)}(\cdot) \right).
    \end{equation*}
    For the base case, consider the $R_1$ outputs of the first hidden function $\hiddenTT{1}_{r_1}$. They are i.i.d. with mean
    \begin{equation*}
        \E\left[\hiddenTT{1}_{r_1}(x_1)\right] = 0.
    \end{equation*}
    and covariance
    \begin{equation*}
        \E\left[\hiddenTT{1}_{r_1}(x_1)\hiddenTT{1}_{r_1}(x_1')\right]
        = \frac{1}{\sqrt{R_1}}{\mat{\varphi}^{(1)}({x_1})}\trans \mat{\Lambda}^{(1)} \mat{\varphi}^{(1)}({x_1}).
    \end{equation*}
    We now consider the $R_2$ outputs of the second hidden function $\hiddenTT{2}_{r_2}$
    \begin{equation*}
        \hiddenTT{2}_{r_{2}} = \sum_{r_{1}=1}^{R_1}  z^{(2)}_{r_{1}, r_{2}} (x_{2}) \hiddenTT{1}_{r_1},
    \end{equation*}
    which are i.i.d. as they are the same function of the $R_1$ i.i.d. outputs of $\hiddenTT{1}_{r_1}(x_1)$. More specifically, their mean and covariance are
    \begin{subequations}
        \begin{align*}
            & \E\left[\hiddenTT{2}_{r_2}\right] = 0, \\
            & \E\left[\hiddenTT{2}_{r_2}(x_2) \hiddenTT{2}_{r_2}(x_2')\right] \\
            =& \frac{1}{\sqrt{R_2}} \prod_{d=1}^{2} {\mat{\varphi}^{(d)}({x_d})}\trans \mat{\Lambda}^{(d)} \mat{\varphi}^{(d)}({x_d}).
        \end{align*}
    \end{subequations}
    Once more by the \gls*{clt} $\hiddenTT{2}_{r_2}$ converges in distribution to the normal distribution with the above first two moments. Since any finite collection of $\{ \hiddenTT{2}_{r_{2}}(\mat{x}_{1:2}),\ldots, \hiddenTT{2}_{r_{2}}(\mat{x'}_{1:2}\}$ will have a joint multivariate normal distribution with the aforementioned first two moments, we conclude that $\hiddenTT{2}_{2}(\mat{x}_{1:2})$ is the \gls*{gp}
    \begin{equation*}
        \hiddenTT{2}_{r_2}\sim\mathcal{GP}\left(0, \frac{1}{\sqrt{R_2}} \prod_{d=1}^{2} {\mat{\varphi}^{(d)}(x_d})\trans\mat{\Lambda}^{(d)} \mat{\varphi}^{(d)}(\cdot) \right),
    \end{equation*}
    which is our base case.
    Hence by induction ${f_{\ff{TT}}(\mat{x})} = \hiddenTT{D}$ converges in distribution as $R_1\rightarrow\infty$, $R_2\rightarrow\infty$, \ldots, $R_{D-1}\rightarrow\infty$ to the \gls*{gp}
    \begin{equation*}
        f_{\ff{TT}}(\mat{x}) \sim \mathcal{GP}\left(0, \prod_{d=1}^D {\mat{\varphi}^{(d)}(x_d)}\trans\mat{\Lambda}^{(d)} \mat{\varphi}^{(d)}(\cdot)\right).
    \end{equation*}
\end{proof}

\subsection{\gls*{gp} of \gls*{tt}-Constrained Kernel Machine in the Simultaneous Limit of the \gls*{tt} Ranks}\label{appendix:thm_mathews}
In \cref{thm:TT} we prove by induction that the \gls*{tt}-constrained kernel machine converges to a \gls*{gp} by taking successive limits of the \gls*{tt} ranks. This result is analogous to the work of \citet{lee_deep_2018}, who prove that for the \glspl*{dnn}, taking sequentially the limit of each layer.
A more practically useful result consists in the convergence in the \emph{simultaneous} limit of \gls*{tt} ranks.

In deep learning \citet[theorem 4]{matthews_gaussian_2018} prove convergence in the context of \glspl*{dnn} over the widths of all layers simultaneously. Said theorem has been employed to prove \gls*{gp} convergence in the context of convolutional neural networks \citep{garriga-alonso_deep_2018} and in the context of \glspl*{dnn} where each weight matrix is of low rank \citep{nait-saada_beyond_2023}.

Seeing the similarity between \gls*{tt}-constrained kernel machines (\cref{eq:f_TT}) and \glspl*{dnn} and the technicality of the proof, similarly to \citep{garriga-alonso_deep_2018,nait-saada_beyond_2023} we draw a one-to-one map between the \gls*{tt}-constrained kernel machines and the \glspl*{dnn} considered in \citet[theorem 4]{matthews_gaussian_2018}. Convergence in the simultaneous limit is then guaranteed by \citet[theorem 4]{matthews_gaussian_2018}.

We begin by restating the definitions of linear envelope property, \glspl*{dnn}, linear envelope property and normal recursion as found in \citet{matthews_gaussian_2018}.
To make the comparison easier for the reader, we change the indexing notation to match the one in this paper.
\begin{definition}[Linear envelope property for nonlinearities \citep{matthews_gaussian_2018}]\label{def:envelope}
    A nonlinearity $t:\mathbb{R}\rightarrow\mathbb{R}$ is said to obey the linear envelope property if there exist $c,l\geq 0$ such that the following inequality holds
    \begin{equation}
        \lvert t(u) \rvert < c+l \lvert u \rvert \ \forall u \in \mathbb{R}.
    \end{equation}
\end{definition}

\begin{definition}[Fully connected \gls*{dnn} \citep{matthews_gaussian_2018}]\label{def:DNN}
A fully connected deep neural with one-dimensional output and inputs $\mat{x}\in\mathbb{R}^{R_0}$ is defined recursively such that the initial step is
    \begin{equation}\label{eq:NN1}
        h_{r_1}^{(1)}(\mat{x}) = \sum_{r_0=1}^{R_0} z^{(1)}_{r_1,r_0} x_{r_0} + b_{r_1}^{(1)},
    \end{equation}
    the activation step by nonlinear activation function $t$ is given by
    \begin{equation}\label{eq:NN2}
        g_{r_d}^{(d)} = t(f_{r_d}^{(d)}),
    \end{equation}
    and the subsequent layers are defined by the recursion
    \begin{equation}\label{eq:NN3}
        h_{r_{d+1}}^{(d+1)} = \sum_{r_d=1}^{R_d} z^{(d+1)}_{r_{d+1},r_d} g_{r_d}^{(d)} + b_{r_{d+1}}^{d+1},
    \end{equation}
    so that $h^{(D)}$ is the output of the network. In the above, $\mat{Z}^{(d)}\in\mathbb{R}^{R_{d-1}\times R_d}$ and $\mat{b}^{(d)}\in\mathbb{R}^{R_d}$ are respectively the weights and biases of the $d$-th layer.
\end{definition}

\begin{definition}[Width function \cite{matthews_gaussian_2018}]\label{def:width_func}
    For a given fixed input $n\in\mathbb{N}$, a width function $v^{(d)}:\mathbb{N}\rightarrow\mathbb{N}$ at depth $d$ specifies the number of hidden units $R_d$ at depth $d$.
\end{definition}

\begin{lemma}[Normal recursion \citep{matthews_gaussian_2018}]\label{lemma:normal_recursion}
    Consider $z^{(d)}_{r_{d-1},r_d}\sim \mathcal{N}(0,C^{(d)}_w)$ and $b_{r_d}^{(d)}\sim\mathcal{N}(0,C_b^{(d)})$.
    If the activations of the $d$-th layer are normally distributed with moments
    \begin{align}
        & \E\left[h_{r_d}^{(d)}\right] =0 \\
        & \E\left[h_{r_d}^{(d)} h_{r_d}^{(d)}\right] = K(x,x'),
    \end{align}
    then under recursion \cref{eq:NN2,eq:NN3}, as $R_{d-1}\rightarrow\infty$, the activations of the next layer converge in distribution to a normal distribution with moments
    \begin{align}
        & \E\left[h_{r_{d+1}}^{(d+1)}\right] =0 \\
        & \E\left[h_{r_{d+1}}^{(d+1)} h_{r_{d+1}}^{(d+1)}\right] =  C_w^{(d+1)} \E_{(\epsilon_1,\epsilon_2)\sim\mathcal{N}(0,K)}\left[t(\epsilon_1)t(\epsilon_2)\right]+C_b^{(d+1)}.
    \end{align}
\end{lemma}
We can now state the major result in \citet{matthews_gaussian_2018}.
\begin{theorem}[\gls*{gp} in the simultaneous limit of fully connected \glspl*{dnn} \citep{matthews_gaussian_2018}]\label{thm:mathews}
    Consider a random \gls*{dnn} of the form of \cref{def:DNN} obeying the linear envelope condition of \cref{def:envelope}. Then for all sets of strictly increasing width functions $v^{(d)}$ and for any countable input set $\{\mat{x},\ldots,\mat{x'}\}$, the distribution of the output of the network converges in distribution to a \gls*{gp} as $n\rightarrow\infty$. The \gls*{gp} has mean and covariance functions given by
    the recursion in \cref{lemma:normal_recursion}.
\end{theorem}

\begin{corollary}[\gls*{gp} in the simultaneous limit of \gls*{tt}-constrained kernel machines]\label{cor:simultaneous}
    Consider a random \gls*{tt}-constrained kernel machine of the form of \cref{def:ttkm} obeying the linear envelope condition of \cref{def:envelope}. Then for all sets of strictly increasing width functions $v^{(d)}$ and for any countable input set $\{\mat{x},\ldots,\mat{x'}\}$, the distribution of the output of the network converges in distribution to a \gls*{gp} as $P\rightarrow\infty$. The \gls*{gp} has mean and covariance functions given by
    the recursion in \cref{lemma:normal_recursion} and stated in \cref{thm:TT}.
\end{corollary}
\begin{proof}
    When examining \cref{def:DNN} and comparing it with \cref{def:ttkm} it becomes clear that both models are similar. 
    In the special case of involving linear activation function and zero biases, the models are structurally identical if one considers unit inputs $x=1$ in \cref{eq:NN1}.
    The normal recursion in \cref{lemma:normal_recursion} is satisfied by \gls*{tt}-constrained kernel machines, as we have that
    \begin{align*}
        t(u) & \isdef u \ \forall u\in\mathbb{R}, \\
        C_b^{(d+1)} & \isdef 0, \\
        C^{(d+1)} & \isdef  \frac{1}{\sqrt{R_{d}R_{d+1}}}{\mat{\varphi}^{(d)}(x_d)}\trans\mat{\Lambda}^{(d)} \mat{\varphi}^{(d)}(x_d'), \\
        K & \isdef \frac{1}{\sqrt{R_d}} \prod_{p=1}^d {\mat{\varphi}^{(p)}(x_p)}\trans\mat{\Lambda}^{(p)} \mat{\varphi}^{(p)}(x_p') \\
        \E_{(\epsilon_1,\epsilon_2)\sim\mathcal{N}(0,K)}\left[t(\epsilon_1)t(\epsilon_2)\right] & \isdef K.    
    \end{align*}
    Hence by \cref{thm:mathews}, for all sets of strictly increasing width functions $v^{(d)}$ and for any countable input set $\{\mat{x},\ldots,\mat{x'}\}$, the distribution of the output of the network converges in distribution to a \gls*{gp}, fully specified by the output of the normal recurision in \cref{lemma:normal_recursion}, which equals the \gls*{gp} in \cref{thm:TT}.
\end{proof}


\end{document}